%% file: learningDLs.tex
\documentclass[runningheads]{llncs}
\usepackage{graphicx}
\usepackage{tikz}
\input{tik.tex}

\pgfdeclarelayer{bg}    
\pgfsetlayers{bg,main,nodelayer}  

\definecolor{Black}  {RGB}{0,0,0}
\input{misc0.tex}

\newcommand{\examples}{\ensuremath{\Emc}\xspace} 
\newcommand{\hypothesisSpace}{\ensuremath{\Lmc}\xspace} 
\newcommand{\target}{\ensuremath{t}\xspace}
\newcommand{\lab}[2]{\ensuremath{\ell_{#2}(#1)}\xspace}

\newcommand{\learnertm}{\ensuremath{L}\xspace}
\newcommand{\teachertm}{\ensuremath{T}\xspace}
\newcommand{\sma}{{\ensuremath{{\boldsymbol{\sigma}}}}}

\newcommand{\e}{\ensuremath{e}\xspace} 
\newsavebox{\spacebox}
\begin{lrbox}{\spacebox} 
\verb*! !
\end{lrbox}
\newcommand{\blank}{\usebox{\spacebox}}
\begin{document}
\title{On the Complexity of Learning \\ Description Logic Ontologies\thanks{Supported by the University of Bergen.}}
%
%
\author{Ana Ozaki\orcidID{0000-0002-3889-6207}}
\authorrunning{A. Ozaki}
%
\institute{University of Bergen, Norway  \\
\email{Ana.Ozaki@uib.no}}
\maketitle              
\begin{abstract}
Ontologies are a popular way of representing
 domain knowledge, 
in particular, knowledge in domains  related to life sciences.
(Semi-) automating 
the process of building an ontology has attracted researchers from different 
communities into a field called ``Ontology Learning''. 
We provide a formal specification of the exact and the
probably approximately correct learning models from computational learning theory.
Then, we recall from the literature  complexity results for learning lightweight description logic (DL) ontologies in these models.
Finally, we highlight  
other approaches 
proposed in the literature  for 
learning  DL ontologies. 

\keywords{Description Logic  \and Exact Learning \and Complexity Theory}
\end{abstract}
\section{Introduction}

Ontologies
have been used to build  concept and role hierarchies mapping and integrating
the vocabulary of data sources, to 
model definitional sentences in a domain of interest,
to support the inference of facts in knowledge graphs, among others. 
However, modelling an ontology that captures in a precise
and clear way the relevant knowledge of a domain can be quite time-consuming.
To imagine this, consider the task of writing a text on a particular topic.
The writer needs to select the right words, think about their meaning,
 delineate the scope, and the essential information
she or he wants to convey. The knowledge of the writer
needs to be clearly represented in a language, using the vocabulary and
the constructs available in it. 
In this way, building an ontology
can be seen as a similar process
but often there is the additional challenge that
the ontology needs to capture the knowledge of a domain in which the
`writer'---an ontology engineer---is not familiar with.
The information in an ontology may need  to be
validated by
a domain expert. 
Because building and maintaining ontologies are demanding tasks, 
several researchers have worked on developing 
theoretical results
and practical tools  for supporting 
this process~\cite{lehmann2014perspectives,anasurvey}. 

Here we consider two learning models that were applied  to model the
process of building an ontology.
The first is the exact learning model~\cite{angluinqueries}. 
In the exact learning model, a learner attempts to communicate
with a teacher in order to identify
an abstract target concept.
When instantiating the exact learning model to capture the process of building an ontology,
one can consider that the teacher is a domain expert who knows
the domain but cannot easily formulate it as an ontology~\cite{KLOW18}.
On the other side, the role of the learner is played by an ontology engineer.
The abstract target that the ontology engineer wants to identify is an ontology
that reflects the knowledge of the domain expert (the teacher). 
%

Although the teacher in the exact learning model is often described as a human (potentially a domain expert), it can also be
  a batch of examples~\cite{ariasetal}, an artificial neural network~\cite{DBLP:conf/nips/WeissGY19,DBLP:conf/icml/WeissGY18},
 a Tsetlin machine~\cite{DBLP:journals/corr/abs-1804-01508} or any another formalism that can be used
 to simulate the teacher. 
The most studied communication protocol is based on 
\emph{membership} and \emph{equivalence} queries. 
%
A membership query gives to the learner the ability to formulate an example and ask for its classification (``does X hold in the domain?''). 
This  mode of learning is  called \emph{active learning}.
   In an equivalence query, the learner asks whether 
a certain hypothesis is equivalent to the target.
There are various results in the literature showing that
the combination of these two types of queries allows
the learner to correctly identify the target concept in polynomial
time, with hardness results for the case in which one of the two queries is disallowed~\cite{angluinqueries,DBLP:journals/ml/AngluinFP92,DBLP:journals/ml/Angluin90,DBLP:journals/toct/HermoO20,DBLP:journals/jacm/PittV88}.


The second model we study is the classical probably approximately correct (PAC) learning model~\cite{Valiant}.
In the PAC  model, a learner receives classified examples according 
to a probability distribution. Then, the learner attempts to build a hypothesis
that is consistent with the examples. This mode of learning is called \emph{passive learning}
because, in contrast with the active learning mode, the learner has no control of
which examples are going to be classified.
One can instantiate this model to the problem of learning ontologies
by considering that 
the ontology engineer starts attempting to collect information about
the domain at random (instead of interacting with the domain expert).  One can also consider the case in which,
in addition to the search at random,
the ontology engineer can pose membership queries to the  expert.
%
Equivalence queries are not considered in the PAC model because of a general
result showing that learners that can pose equivalence queries to learn a certain target
are also able to accomplish this task within the PAC model~\cite{angluinqueries} (the combination
is not interesting because one problem setting is `easier' than the other, more details
in Subsection~\ref{sec:learnability}). 


Having these models in mind, we first present
the syntax and the semantics of the ontology language \ELH~\cite{dlhandbook,baader_horrocks_lutz_sattler_2017}.
This is a prototypical lightweight ontology language based on description logic (DL).
Then,  
we formalise the exact and the PAC learning models using
 notions from the theory of computation. 
 We recall from the literature complexity results for learning lightweight
DL ontologies in these models and provide intuitions about these results. 
 Finally, we point out 
 other approaches that have been applied for learning DL ontologies. 


\section{Description Logic}\label{sec:el}
We introduce \ELH~\cite{dlhandbook,baader_horrocks_lutz_sattler_2017}, a classical lightweight
DL which features
  existential quantification ($\exists$) and conjunction ($\sqcap$).
Let $\NC$, $\NR$, and \NI be countably infinite and disjoint 
sets of \emph{concept} and \emph{role} names.
An \ELH ontology (or \emph{TBox}) is a finite set 
of \emph{role inclusions} (RIs) $r\sqsubseteq s$ with $r,s\in\NR$ and
 \emph{concept inclusions} (CIs)
$C\sqsubseteq D$ 
with $C,D$   \EL \emph{concept expressions} built according to the 
  rule
\begin{gather*}
C,D::= A \mid \top \mid C\sqcap D\mid \exists r.C
\end{gather*}
with $A\in\NC$ and $r\in\NR$. 
An \ELH \emph{TBox} is a finite set of RIs and CIs $C\sqsubseteq D$, 
with $C,D$ being \EL concept expressions.
We denote by $\ELHlhs$ and $\ELHrhs$
the fragments of \ELH that allow only a concept name on the right-hand side
and on the left-hand side, respectively. That is, \ELHlhs
is the language that allows \emph{complex} \EL expressions on the left-hand side
and the same idea applies for \ELHrhs.
An \EL TBox is an \ELH TBox that does not have RIs.
We may write $C\equiv D$ 
as a short hand for having   both $C\sqsubseteq D$ and $D\sqsubseteq C$.
An \emph{assertion} is an expression of the form
$r(a,b)$ or of the form $A(a)$, where $A\in\NC$, $r\in\NR$, and $a,b\in\NI$.
An \emph{ABox} is a finite set of assertions. An \ELH \emph{instance query}
(IQ) is of the form $C(a)$ or $r(a,b)$ with $C$  an
\EL concept expression, $r\in\NR$, and $a,b\in\NI$.

We now present the usual semantics of \ELH,
which is based on \emph{interpretations}. 
An interpretation \Imc is a pair $(\Delta^\Imc,\cdot^\Imc)$ 
where $\Delta^\Imc$ is a non-empty set, called the \emph{domain of \Imc}, and $\cdot^\Imc$ is a function 
mapping each $A\in\NC$ to a subset $A^\Imc$ of $\Delta^\Imc$,
 each $r\in\NR$ to a subset $r^\Imc$ of $\Delta^\Imc\times \Delta^\Imc$, and
 each $a\in\NI$ to an element in $\Delta^\Imc$. 
The function $\cdot^\Imc$ extends to arbitrary \EL
concept expressions as follows: 
\begin{align*}
(\top)^\Imc &:= {}  \Delta^\Imc\\
 (C\sqcap D)^\Imc &:= {}  C^\Imc\cap D^\Imc \\
(\exists r.C)^\Imc &:= {}  \{d\in\Delta^\Imc\mid \exists e\in C^\Imc \text{ such that } (d,e)\in r^\Imc\}
\end{align*}
The interpretation
\Imc \emph{satisfies} an RI $r\sqsubseteq s$ 
iff 
$R^\Imc\subseteq s^\Imc$. 
It satisfies a CI $C\sqsubseteq D$ 
iff 
$C^\Imc\subseteq D^\Imc$. 
It satisfies a TBox \Tmc 
iff 
\Imc satisfies all RIs and CIs in \Tmc.
We write $\Imc\models \alpha$ if
\Imc satisfies an RI, a CI, or a TBox $\alpha$. 
A TBox \Tmc \emph{entails} a CI $\alpha$ (or a TBox $\Tmc'$), written
$\Tmc\models \alpha$ (or  $\Tmc\models\Tmc'$), iff all interpretations 
satisfying \Tmc also satisfy $\alpha$ (or $\Tmc'$).
Two TBoxes \Tmc and $\Tmc'$ are \emph{equivalent}, written $\Tmc\equiv\Tmc'$, iff 
$\Tmc\models\Tmc'$ and $\Tmc'\models\Tmc$.
These notions can be adapted as expected for
defining satisfaction and entailment of  assertions, ABoxes, and IQs~\cite{dlhandbook}. 
Given a TBox \Tmc, the \emph{signature} $\Sigma_\Tmc$ of \Tmc
is the set of concept and role names occurring in it.

 \begin{example}
Although \ELH is a very simple language, it is already  useful
to represent certain kinds of static knowledge. 
One can express 
in \ELH 
that `Penicillamine nephropathy is a renal disease' with the CI\footnote{This example
follows  modelling guidelines  and terms found in the ontology from the medical domain SNOMED CT~\cite{snomed-rt-97}. }:
$${\sf PenicillamineNephropathy}\sqsubseteq {\sf RenalDisease}$$
The definitional sentence
`Penicillamine nephropathy is a renal disease of the kidney structure caused 
by penicillamine' can be expressed as:
$${\sf PenicillamineNephropathy}\equiv {\sf RenalDisease} \ \sqcap $$$$
\exists {\sf findingSite}.{\sf KidneyStructure} \sqcap 
\exists {\sf causativeAgent}.{\sf Penicillamine}.$$
An important advantage of representing domain knowledge in an ontology
is that ambiguities found in natural language can be removed. For example,
one can distinguish when `is' should mean a subset relation (represented syntactically
with `$\sqsubseteq$') from when it means  an equivalence (represented with the symbol  `$\equiv$'). 
\end{example}

In the next section, we introduce  notions from the theory of computation that
are relevant to define learnability and complexity classes in
the exact and PAC learning models. 

\section{The Complexity of Learning}

Complexity classes are defined in terms of a  model of computation, a
type of problem, and bounds on the resources (usually time and memory) needed to solve a problem~\cite{sipser-complexity}.
In this section, we formally define complexity classes for \emph{learning problems} (Subsection~\ref{sec:learnability}).
Before that, we define a general model of computation that represents the communication of
 a learner and a teacher via queries (Subsection~\ref{sec:comp}). 
This model  can be specialised
for learning problems in the exact and the PAC learning models.
For the exact learning model, 
we assume that the learner can pose membership and equivalence queries.
In the PAC learning model, the learner can pose sampling queries.
We describe the queries in detail in Subsection~\ref{sec:learning-problems}.

\subsection{Model of Computation}\label{sec:comp}

The main advantage of  defining the model of computation
is that this opens the possibility of analysing
the  learning phenomenon in light of 
the theory of computation.
Our model of computation for learning problems
is based on \emph{learning systems}~\cite{DBLP:conf/icalp/Watanabe90}. 
Learning systems are formulated using the notion of a pair $(\learnertm,\teachertm)$ of 
  multitape Turing machines  (MTM),
one for the learner, $\learnertm$, and one for the teacher, $\teachertm$.
There are 
four kinds of tape:
\begin{itemize}
\item $\learnertm,\teachertm$ share a read-only \emph{input tape};
\item $\learnertm,\teachertm$ share a read-write \emph{communication tape};
\item $\teachertm$ has a read-only tape, called \emph{oracle tape}, not accessed by $\learnertm$; and  
\item $\learnertm$ has a write-only \emph{output tape}, not accessed by $\teachertm$.
\end{itemize}

Intuitively, the computation of the two MTMs represents the
interaction of the learner and the teacher via queries, where 
posing a query to 
the teacher means writing down 
the input of the query in the communication tape and entering  the
corresponding query state. 
Then the teacher 
computes the answer,  writes it in the communication tape (if the computation of the answer terminates), and enters  the
corresponding answer state. 
The learner 
reads the answer in the communication tape and continues its computation.
This process may continue forever or 
halt when the learner writes
its final hypothesis in the output tape and enters  the final state. 
It is  assumed that
 the teacher
 never enters  the final state (only the learner
 can enter  the final state).

In the definition of a learning system $(\learnertm,\teachertm)$,
we consider that $\learnertm$ is a  deterministic MTM (DMTM) with three
tapes (input, communication, and output tapes)
and that the set of states contains special elements called \emph{query states},
one for each type of query. 
The teacher $\teachertm$ is a non-deterministic MTM (NMTM) with three
tapes (input, communication, and oracle tapes) and the set of states contains
 special elements called \emph{answer states}, one for each type of query.
We describe the types of queries for the exact and PAC learning models in Subsection~\ref{sec:learning-problems}.

A DMTM with $k$ tapes can be defined as a tuple $\Mmc = (Q,\Sigma, \Theta, q_0,q_{\sf f})$ where:
$Q$ is a finite set of \emph{states};
$\Sigma$ is a finite \emph{alphabet} containing the \emph{blank symbol}
$\blank$; 
$\Theta:(Q\setminus\{q_{\sf f}\})\times\Sigma^k\rightarrow Q\times \Sigma^k\times\{l,r\}^k$ is the
\emph{transition function}; 
and 
$\{q_0,q_{\sf f}\}\subseteq Q$ are the \emph{initial} and 
\emph{final} states.
The expression $\Theta(q,a_1,\ldots,a_k)=(q',b_1,\ldots,b_k,D_1,\ldots,D_k)$,
with $D_i\in\{l,r\}$, means that if \Mmc is in state $q$ and heads
$1$ through $k$ are reading the symbols
 $a_1$ through $a_k$ (resp.) then \Mmc goes to state $q'$,
 the symbols $b_1,\ldots,b_k$ are written in 
 tapes $1,\ldots,k$ (resp.) and each head $i$ moves  to the direction corresponding to $D_i$. 
An NMTM is defined
in the same way as a DMTM except that $\Theta:(Q\setminus\{q_{\sf f}\})\times\Sigma^k\rightarrow
\Pmc(Q\times \Sigma^k\times\{l,r\}^k)$. 
That is,  $\Theta(q,a_1,\ldots,a_k)$ is now  a
\emph{set} of expressions of the form $(q',b_1,\ldots,b_k,D_1,\ldots,D_k)$.



A \emph{configuration} of a MTM with $k$ tapes is a $k$-tuple 
$(w_1qw'_1,\ldots,w_kqw'_k)$
with $w_i,w'_i \in
\Sigma^\ast$ and $q \in Q$, meaning that the tape $i$ contains the
word $w_iw'_i$, the machine is in state $q$ and the head is on the position
of the left-most symbol of $w'_i$.
The notion of \emph{successive configurations} is defined as expected
  in terms of the transition relations of $\learnertm$ and $\teachertm$. 
  Whenever $\learnertm$ enters  a query state,
  the transition relation of 
  $\teachertm$ is used to define a successive configuration (it may not be unique due to non-determinism of $\teachertm$) 
  and whenever $\teachertm$ enters again in an answer state
  then  
  the transition relation of $\learnertm$ defines the (unique) successive configuration
  ($\learnertm$ resumes its execution). 
  %
  %
 %
 A \emph{computation} of   $(\learnertm,\teachertm)$
 on an input word $w_0$ is a tree whose paths are sequences of successive configurations 
 $\alpha_0, \alpha_1, \ldots$, where $\alpha_0=q_0 w_0$ is the
 \emph{initial configuration} for the input $w_0\in (\Sigma
 \setminus \{\blank\})^\ast$ and $q_0$ is the initial state of $\learnertm$.
 The branches of the tree correspond to the different
 possibilities for  $\teachertm$ to move from one state to another.

 The model of computation that we presented can be generalised to the case in which there are
 multiple learners and multiple teachers. For our purposes, it suffices
 to consider only one learner and one teacher.
In the following, we
explain how the computational model we described can be tailored
to the exact and the PAC learning models, as well as some variants of these models. 

\subsection{Learning Frameworks and Queries}
\label{sec:learning-problems}

%

To define the learnability in the exact and PAC models (Subsection~\ref{sec:learnability}),
we use the notion of a \emph{learning framework} and three types of queries (membership and equivalence queries
for the exact learning model and sample queries for the PAC learning model). 
A learning framework $\Fmf$ is a triple 
$(\examples, \hypothesisSpace, \mu)$ where
\begin{itemize}
\item $\examples$ is a set of examples, 
\item $\hypothesisSpace$ is a set of concept representations\footnote{In  Machine Learning,
 a concept representation is a way of representing a set of examples. 
This  differs from the notion of a concept in  DL.}, called \emph{hypothesis space},
\item and $\mu$ is a function that maps each element of \Lmc to a 
set of 
examples in $\examples$.
\end{itemize}

We call \emph{target} 
a fixed but arbitrary element   of
 $\hypothesisSpace$ that the learner
wants to acquire. 
A \emph{hypothesis} is an element   of $\hypothesisSpace$
that represents the `idea' of  the learner about
the target. This element is often updated during the computation
of a learning system $(\learnertm,\teachertm)$
until 
$\learnertm$
reaches its final state (if ever).
Given a target $t\in\Lmc $, we say that an example $\e$ is \emph{positive}
for $t$ if $\e \in \mu(t)$, and \emph{negative} otherwise.
Given a hypothesis $h$ and a target $t$ in $\Lmc$  and an example $\e\in\examples$,
we say that $\e$ is a \emph{counterexample} for $t$ and $h$
if $\e \in \mu(t)\oplus\mu(h)$ (where $\oplus$ denotes the symmetric difference). We may omit `for $t$' and 
`for $t$ and $h$' if this is clear from the context. 

\begin{remark}
Given a DL  $\Lmf$, we denote by $\Fmf(\Lmf)$ the learning
framework
$(\examples, \hypothesisSpace, \mu)$
where $\examples$ is the set of CIs and RIs that can be formulated in \Lmf (using symbols from \NC and \NR),
\hypothesisSpace is the set of all \Lmf TBoxes, and 
 $$\mu(\Tmc)=\{\alpha \mid \Tmc \models \alpha, \text{ with $\alpha$ a CI or an RI in }\Lmf\}.$$ 
This  setting is called  \emph{learning  
from  entailments}.
In the learning framework $\Fmf(\ELHrhs)=(\examples, \hypothesisSpace, \mu)$
we have that 
$\Tmc=\{A\sqsubseteq \exists r.A\}\in\Lmc$ and,
for all $n\in\mathbb{N}$, the CI $A\sqsubseteq \exists r^n.A$ is in $\mu(\Tmc)$,
where $\exists r^{n+1}.A:=\exists r.\exists r^{n}.A$ and $\exists r^1.A:=\exists r.A$.
\end{remark}

One could define a more general notion of a learning framework,
where the hypothesis space for the hypothesis of the learner differs from
the hypothesis space that contains the target.
Also, the mapping function $\mu$ could be adapted to represent
fuzzy sets of examples.
We keep
the version introduced above because it is general enough
for our purposes and 
covers  classical problems
in the literature~\cite{angluinqueries,Valiant}. 
%
%
%
We now describe in detail the queries that the learner can pose
and how the teacher answers these queries.
Consider a learning framework $\Fmf=(\examples, \hypothesisSpace, \mu)$,
and a learning system $(\learnertm,\teachertm)$ with 
 a fixed but arbitrary  target $t\in\Lmc$ in the oracle tape.
\begin{itemize}
\item A \textbf{membership query} happens whenever $\learnertm$ writes an example $\e$ in the communication
tape and enters the membership query state. In this case,  
$\teachertm$ resumes the execution and (if the computation terminates) writes
`yes' 
in the communication tape if $\e \in \mu(t)$, otherwise, it writes `no'
(assume such answers can be formulated using symbols from the alphabets of $\learnertm$ and $\teachertm$). 
\item An \textbf{equivalence query} happens whenever $\learnertm$ writes a hypothesis  $h \in \Lmc$ in the communication
tape and enters  the equivalence query state. 
The teacher $\teachertm$ resumes its execution and (if the computation terminates) writes
some  $\e \in \mu(t)\oplus\mu(h)$ in the communication tape, or `yes'
if $\mu(t)=\mu(h)$.
\item  A \textbf{sample query} happens whenever $\learnertm$ enters  the sample query state.
In this case, the teacher $\teachertm$ resumes its execution and (if the computation terminates) writes
some $(\e,\lab{\e}{t})$ in the communication tape, where
the
choice of $\e\in\examples$ is according to a fixed but arbitrary
probability distribution on $\examples$ (unknown to the learner)
and $\lab{\e}{t}=1$, if $\e\in\mu(t)$, and $0$ otherwise. 
\end{itemize}
We  write $(\learnertm_\Fmf,\teachertm_\Fmf(t))$  to indicate that 
$t$ is in the oracle tape and queries/answers are as just described for a
learning framework \Fmf 
(we may omit the subscript \Fmf if this is clear from the context). 
For some learning frameworks 
and some types of queries, it can be assumed that
the computation of answers by the teacher 
always terminates independently of which $t\in\Lmc$ happens to be in the oracle tape.
One example is when the $\mu$ function encodes the entailment relation
and the entailment problem of the logic represented 
in \Lmc
 is decidable (e.g., entailment in \ELH is decidable in polynomial time~\cite{BBL-EL}).
 However, if the entailment problem is undecidable 
 this assumption cannot be made  
 independently of the content of the oracle tape (e.g., entailment in first-order logic).
%

Even if there is a teacher that always terminates depending on the content of the
oracle tape, naturally,
one cannot assume that all of them will terminate.
So we define the following notion.
Let $\teachertm(t)$ be a teacher with $t\in\Lmc$ in the oracle tape.
We say that  $\teachertm(t)$ is \emph{terminating for membership queries}
if for
 every   possible membership query (within a learning framework) 
the teacher $\teachertm(t)$ always terminates the computation of the answer. 
This notion can be easily adapted for other types of queries. 

The multiple ways of
choosing 
$\e \in \mu(t)\oplus\mu(h)$ in an equivalence query
and an example $\e\in\examples$ in a sample query 
is captured by the non-determinism of $\teachertm$ (see Subsection~\ref{sec:comp}). 
For representing sample queries, one can
consider the special case in which the 
NMTM is a multitape probabilistic Turing machine~\cite{sipser-complexity}.
%
We may write $\teachertm_\Dmc$ to indicate that, whenever a sample query
is posed by $\learnertm$, we have that $\teachertm$
chooses an example according to the (same) probability distribution \Dmc,
with the events of drawing examples being mutually independent (see e.g.~\cite{Shalev-Shwartz:2014:UML:2621980} for more details on sample queries
and~\cite{anasurvey} for a presentation using this notation).






\subsection{Learnability and Complexity Classes}
 \label{sec:learnability}

We are   ready to define the notion of learnability
and complexity classes for learning problems. 
We write $Y\in(\learnertm,\teachertm(t))(X)$ if there is a finite computation of
the learning system  $(\learnertm,\teachertm(t))$ with $X$ in the input tape,
$t$ in the oracle tape,
 and the content written by the learner in the output tape, $\learnertm$,
 is $Y$.
We first define learnability for the exact learning model. 

 Let
$\Fmf = (\examples, \hypothesisSpace, \mu)$ be
a learning framework. 
Assume that the learner can pose membership and equivalence queries
and these are truthfully replied by the teacher, as described in Subsection~\ref{sec:learning-problems}.
We say that \Fmf is \emph{exactly learnable}
if there is 
a learner $\learnertm$
such that, for every $\target\in\hypothesisSpace$, there is a terminating $\teachertm(t)$ (for
membership and equivalence queries). 
Moreover,
\begin{itemize}
\item 
 every  learning system $(\learnertm,\teachertm'(t))$ with a terminating $\teachertm'(t)$
 halts and every $h\in(\learnertm,\teachertm'(t))(\Sigma_t)\cap\hypothesisSpace$
satisfies  
 $\mu(h) = \mu(\target)$, 
where $\Sigma_t$ is the signature of $t$. 
\end{itemize}
If the number of steps made by $\learnertm$ in each path of the computation tree is always 
bounded by a polynomial $p(|\target|,|\e|)$, where $\target\in \hypothesisSpace$
is the target and $\e \in \examples$ is the largest counterexample written
 so far in the communication tape by $\teachertm'(t)$ (in the corresponding path), 
then \Fmf is \emph{exactly learnable in polynomial time}.

We denote by $\XL({\sf MQ},{\sf EQ})$ and
$\PTimeL({\sf MQ},{\sf EQ})$ the classes of all learning frameworks
that are, respectively,
exactly learnable
and
exactly learnable in polynomial time with membership and equivalence queries.
One can easily adapt this notation
to the case in which the learner is allowed to make an exponential number of steps,
denoted $\ExpTimeL({\sf MQ},{\sf EQ})$, or to the case in which 
the learner can only pose one type of query.
For  representing this, we simply drop ${\sf MQ}$ or ${\sf EQ}$ from the class name (e.g., $\PTimeL({\sf EQ})$ is the class of all learning
frameworks that are exactly learnable in polynomial time with only equivalence queries).
One can also consider other types of queries, such as subset and superset
queries~\cite{angluinqueries}, or queries that take into account
the history of previous queries~\cite{ozakitroquard}.
It follows from these definitions that
$ \PTimeL({\sf MQ},{\sf EQ})\subseteq \ExpTimeL({\sf MQ},{\sf EQ})\subseteq\XL({\sf MQ},{\sf EQ})$.

We now define learnability in the PAC model.
 Let
$\Fmf = (\examples, \hypothesisSpace, \mu)$ be
a learning framework. 
Assume that the learner can pose sample queries
and these are replied by the teacher as  in Subsection~\ref{sec:learning-problems}.
The goal is to build a hypothesis such that `with high probability  there
is not much difference between the hypothesis and the target'.
A parameter $\epsilon$ quantifies the error of the hypothesis w.r.t. the target (how different they are).
Another parameter $\delta$ is used to quantify the confidence of meeting the error requirement (whether this has 
high probability). 
Both parameters are real numbers ranging between $0$ and $1$. 
Formally, we say that 
$\Fmf$
is \emph{PAC learnable}
if there is a function $f : (0, 1)^{2} \to \mathbb{N}$ 
and a  learner $\learnertm$
such that, for every $(\epsilon, \delta) \in (0, 1)^2$, 
every probability distribution $\prob$ on $\examples$,
 and every target
$\target \in \hypothesisSpace$, there is a terminating $\teachertm_\Dmc(t)$ (for sample queries). Moreover, 
\begin{itemize}
 \item 
 every 
 $(\learnertm,\teachertm'(t)_\Dmc)$ with a terminating $\teachertm'(t)_\Dmc$ halts
after \learnertm poses $m \geq f(\epsilon, \delta)$ samples queries 
 and, with probability at least $(1 - \delta)$ (over the choice of  sets of $m$ examples),
 $h\in(\learnertm,\teachertm'(t)_\Dmc)(\Sigma_t)\cap \Lmc$ 
satisfies  $\prob(\mu(h) \oplus \mu(t)) \leq \epsilon$. 
\end{itemize}
If
the number of steps made by $\learnertm$ in each path of the computation tree is always
bounded by
a polynomial function $\poly(|\target|, |\e|, 1/\epsilon, 1/\delta)$,
where
$\e$ is the largest example written in the communication tape by $\teachertm'(t)_\Dmc$ (in the corresponding path),
then 
 $\Fmf$
is \emph{PAC learnable in polynomial time}.
We can easily extend these notions to the case in which
the learner can also pose membership queries (with a terminating teacher for
both sample and membership queries). 
We denote by $\PL$ and
$\PTimePL({\sf SQ})$ the classes of all learning frameworks
that are, respectively, PAC learnable
and
PAC learnable in polynomial time with sample queries.
Also, we write 
$\PTimePL({\sf MQ},{\sf SQ})$ for the case
the learner can also pose membership queries. 

\begin{remark}\label{remark}
There is an important difference between the polynomial bound
for the exact and the PAC learning models. In the exact model, $\e$ is the largest counterexample
written  \emph{so far} by the teacher in the  path of computation,
while in the PAC model
$\e$ is the largest example written by the teacher (at any point of
the  path).
The more strict requirement of the exact model
is to avoid a loophole in the definition~\cite{DBLP:journals/ml/Angluin90}.
Whenever the hypothesis of the learner is not equivalent,   the teacher needs to provide
a \emph{counterexample}. Since  this  depends
on both the target and the hypothesis, there could
be a case in which the learner spends an exponential amount
of time (in the size of the target) to discover a hypothesis that would force the teacher to provide
an exponential counterexample.
 Then the learner would have spent a polynomial amount of time in
 the size of the largest counterexample but not in the size of the largest example
given \emph{so far}. This requirement is not necessary in the PAC model
because the teacher does not need to provide an example that depends
on the hypothesis of the learner, so there is no way the learner can `force' the teacher to return
a large example. 
\end{remark}

Theorem~\ref{thm:exact-pac} states that
positive results for the exact learning model with only equivalence queries
are transferable to the PAC model and
this also holds if both models allow membership queries. 

\begin{theorem}\label{thm:exact-pac}
The following holds~\cite{angluinqueries}:
\begin{itemize}
\item $\XL({\sf EQ})\subseteq \PL({\sf SQ})$;
\item $\PTimeL({\sf EQ})\subseteq \PTimePL({\sf SQ})$;
\item $\PTimeL({\sf MQ},{\sf EQ})\subseteq \PTimePL({\sf MQ},{\sf SQ})$. 
\end{itemize}
\end{theorem}

The intuition for Theorem~\ref{thm:exact-pac} is that the learner can pose
sample queries instead of equivalence queries.
By posing sample queries, the learner can obtain a set of classified examples, drawn
according to a fixed but arbitrary  probability distribution. 
If the current hypothesis of the learner misclassifies one of the examples
of this set then the learner has found a counterexample. So it can 
proceed as if it had posed an equivalence query and the teacher
had returned the counterexample.
Otherwise, it is shown in the proof of the theorem that if the sample is large enough
then any hypothesis consistent with the sample satisfies the criteria for PAC learnability. 

For presentation purposes, we have presented only \emph{time} complexity classes for
the exact and the PAC learning models. One can also consider classes that capture other
ways of measuring the resources used by the learner and/or the teacher~\cite{DBLP:journals/ml/AriasK06}.
For example, one can measure the number and size
of queries posed by the learner. In this way, \emph{query} complexity classes could also
be defined~\cite{arias2004exact,KLOW18}.

\section{Learning DL Ontologies}

We provide some and examples and intuitions about the notions presented so far (Subsection~\ref{subsec:alg}).
Then, in Subsection~\ref{subsec:complexity}, we recall results  on learning DL ontologies
in the exact and PAC learning models. 

\subsection{An Example}\label{subsec:alg}

To illustrate the  ideas for learning DL ontologies
in the exact and the PAC learning models, we start by considering
the problem of exactly learning an ontology in a toy language 
that allows only concept inclusions of the form $A\sqsubseteq B$ with $A,B\in\NC$. 

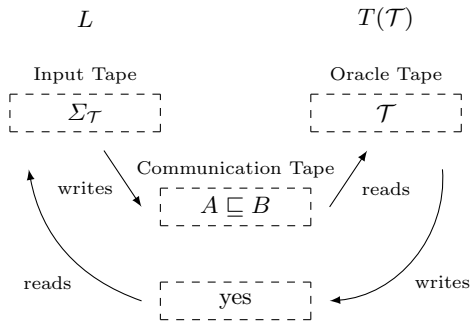
\begin{figure}[h]
\centering
\input{learningsystem.tikz}
\caption{Membership Query in a Learning System}
\label{fig:learningsystem}
\end{figure}

Consider the learning framework $\Fmf_{\sf toy}=(\Lmc,\examples,\mu)$ with \Lmc and \examples being the set of all TBoxes
and the set of all CIs that
can be formulated in the toy language, respectively. The $\mu$ function maps
TBoxes \Tmc in \Lmc to CIs in $\examples$ entailed by \Tmc.  
Suppose the target $\Tmc\in\Lmc$ is $\{A\sqsubseteq B,B\sqsubseteq C\}$ and
let $(\learnertm,\teachertm(\Tmc))$ be a learning system
such that on the input 
$\Sigma_\Tmc=\{A,B,C\}$ (the signature of \Tmc) returns  $\Hmc\equiv\Tmc$.
In symbols, $\Hmc\in (\learnertm,\teachertm(\Tmc))(\Sigma_\Tmc)$.
Clearly, for all $\Tmc'\in\Lmc$, there is a terminating $\teachertm(\Tmc')$
for membership, equivalence, and sample queries. 
%

Figure~\ref{fig:learningsystem} illustrates part of a computation of 
$(\learnertm,\teachertm(\Tmc))$ on
the input $\Sigma_\Tmc=\{A,B,C\}$ where
$\learnertm$ poses the membership query $A\sqsubseteq B\in\mu(\Tmc)$
and receives `yes' as an answer.
A simple strategy for $\learnertm$ is to formulate all
CIs within $\Sigma_\Tmc$ and pose membership queries
with each such CI, one at a time. The CI is added to \Hmc
if, and only if,  the answer is
`yes'. With this strategy, the hypothesis \Hmc computed
by the learner is $\{A\sqsubseteq B,A\sqsubseteq C,B\sqsubseteq C\}$.
At most $|\Sigma_\Tmc|^2$ membership queries are
needed. Thus, $\Fmf_{\sf toy}\in\PTimeL({\sf MQ})$.

\subsubsection{Adding Conjunctions}
Now, consider an extension of the toy language that allows conjunctions of concept names
in CIs. We denote the underlying learning framework as $\Fmf^{\sqcap}_{\sf toy}$.
In this case, the strategy of posing membership queries
for each possible CI formulated within the signature $\Sigma_\Tmc$
of a target \Tmc still terminates (since $\Sigma_\Tmc$ is finite).
However, its does not terminate in polynomial time
in  $|\Sigma_\Tmc|$ because with conjunctions one can
formulate an exponential number of CIs.

In the following, we provide a simple argument showing that
there is no strategy that guarantees polynomial time learnability with only membership queries.
In other words,  $\Fmf^{\sqcap}_{\sf toy}\not\in\PTimeL({\sf MQ})$. 

The main idea   is to
 define a superpolynomial set $S$ of
 TBoxes in this extension of the toy language
 and show that any membership query 
can distinguish at most polynomially many elements of $S$. 
Let $\Sigma = \{A_1,\ldots,A_n,\overline{A}_1,\ldots, \overline{A}_n,M\}$. 
For any sequence $\sma = \sigma^1 \ldots \sigma^n$ 
with $\sigma^i \in \{A_i,\overline{A}_i\}$ the expression $\sma \sqsubseteq M$ 
stands for the CI 
$(\sigma^1 \sqcap \ldots \sqcap   \sigma^n\sqsubseteq M)$. 
For every such sequence $\sma$ (of which there are 
$2^n$ many), consider the
TBox
$\Tmc_\sma$ defined as: 
$$
\begin{array}{rcl}
\Tmc_\sma &=& \left\{
\sma \sqsubseteq M
\right\}\cup\Tmc_0 \; \text{ with }\\[1mm]
 \Tmc_0 &=& \left\{A_i \sqcap \overline{A}_i \sqsubseteq M \mid 1\leq i \leq n\right\}
\end{array}
$$
 
The CI $\sma \sqsubseteq M$ represents a unique binary sequence for each 
$\Tmc_\sma$,   `marked' by the concept name $M$. The CIs 
in  $\Tmc_0$ are shared by all $\Tmc_\sma$ in $S$. 

\begin{lemma}
For any CI $\alpha$ in the extended toy language over $\Sigma$ either:
\begin{itemize}
\item for every $\Tmc_\sma \in S$, we have $\Tmc_\sma\models \alpha$; or 
\item $\Tmc_\sma\models \alpha$, for at most one $\Tmc_\sma \in S$.
\end{itemize}
\end{lemma}

\begin{proof}
Suppose there is $\Tmc_\sma \in S$ such that $\Tmc_\sma\models \alpha$ (otherwise we are done). 
Assume the CI $\alpha$ is $C\sqsubseteq D$.
If $D$ is 
$M$ then either $\alpha$ is a tautology or 
there is no $\Tmc_\sma \in S$ such that $\Tmc_\sma\models \alpha$. 
In both cases, we have that $\Tmc_\sma\models \alpha$ for at most one $\Tmc_\sma \in S$. 
Then, we can assume that $D$ is $M$. 
Regarding   $C$  (the concept on the left side of the CI $\alpha$), 
we make a case distinction:
\begin{itemize}
\item there is $1 \leq i \leq n$ such that $A_i,\overline{A}_i$ are conjuncts in $C$. In this case, 
by definition of $\Tmc_0$, we have that $\Tmc_\sma\models \alpha$, 
for every $\Tmc_\sma \in S$. 
\item there is no $1 \leq i \leq n$ such that $A_i,\overline{A}_i$ are conjuncts in $C$. 
This means that $\Tmc_0\not\models \alpha$.  
If $\alpha$ is of the form $\sma \sqsubseteq M$ then there is exactly 
one $\Tmc_\sma \in S$ such that $\Tmc_\sma\models \alpha$. 
Otherwise, 
there is no $\Tmc_\sma \in S$ 
such that $\Tmc_\sma\models \alpha$.  So, 
$\Tmc_\sma\models \alpha$, for at most one $\Tmc_\sma \in S$. 
\end{itemize}
\end{proof}

Since any 
membership query  can eliminate  only polynomially many 
 elements from $S$ (in our case at most one), 
the learner cannot 
distinguish between the remaining elements from our initial superpolynomial set $S$ in polynomial time. 
Thus, $\Fmf^{\sqcap}_{\sf toy}\not\in\PTimeL({\sf MQ})$.
%
%

This language can be easily translated into propositional Horn.  
It is known that 
propositional Horn expressions are exactly learnable in polynomial time 
if equivalence queries are also allowed~\cite{DBLP:journals/ml/AngluinFP92,DBLP:conf/icml/FrazierP93}. That is, $\Fmf^{\sqcap}_{\sf toy}\in\PTimeL({\sf MQ}, {\sf EQ})$.

\subsubsection{Adding Existentials}
We discuss here a further extension the toy language that also allows existential quantification.
This language  coincides with \EL, defined in Section~\ref{sec:el}. 
Our first observation is that in \EL there is an infinite number of
CIs that can be formulated with a finite signature $\Sigma_\Tmc$ of a target \Tmc.
This happens because existential quantifiers can be nested in concept expressions.
Moreover, due to cyclic references between concepts in an \EL TBox,
an infinite number of CIs can be entailed by a (finite) TBox (see Remark~1).
This means that the strategy of posing membership queries
for each possible CI formulated with the signature $\Sigma_\Tmc$
of a target \Tmc does not terminate in this case.
If equivalence queries are allowed then
one can still enumerate all TBoxes
of size $n$ that can be formulated with  $\Sigma_\Tmc$ (up to logical equivalence)
and ask equivalence queries with such TBoxes, one by one. Then one can increase $n$ until
it reaches the size of \Tmc (which is finite).
This strategy is guaranteed to terminate, although not in polynomial time. 
In the next subsection, we discuss further results for \EL extended with role inclusions (that is, \ELH)
and its fragments $\ELHlhs$ and $\ELHrhs$, introduced in Section~\ref{sec:el}.
  

\subsection{Complexity Results}\label{subsec:complexity}

We now recall from the literature polynomial time complexity results
for learning  DL ontologies
in the exact and the PAC learning models.
Figure~\ref{fig:classes}
illustrates some of these results (some results and complexity classes have been omitted to simplify the presentation). 
 Dashed lines are for the classes associated with the PAC learning model.  
In what follows, we give an overview of the complexity results
and provide additional explanations for the complexity classes.

\begin{figure}[h]
\centering
\includegraphics[width=300pt]{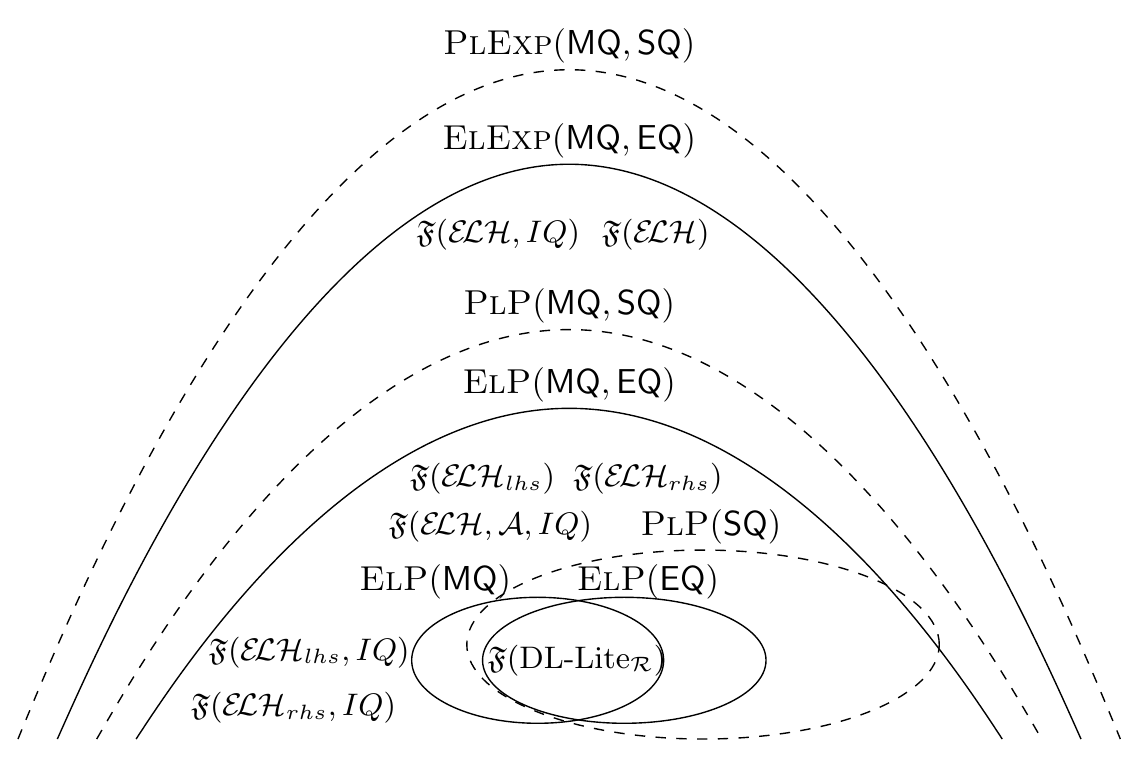}
\caption{Learning Frameworks and Complexity Classes}
\label{fig:classes}
\end{figure}

Konev et al. (2018) have shown that
$\ELH$ (in fact already \EL)
TBoxes are not exactly learnable from entailments in polynomial time
while $\ELHlhs$ and $\ELHrhs$ are polynomially learnable~\cite{KLOW18}.
In symbols, $\Fmf(\ELH)\not\in \PTimeL({\sf MQ},{\sf EQ})$
but $\Fmf(\ELHlhs),\Fmf(\ELHrhs)\in \PTimeL({\sf MQ},{\sf EQ})$. 
Similar results also hold
for a variant of this problem setting where
 the examples are pairs of the form $(\Amc,q)$ (instead of being CIs and RIs),
where \Amc is an ABox and $q$ is an (\ELH) IQ~\cite{DBLP:conf/aaai/KonevOW16}.
In this setting,   $(\Amc,q)$
is a positive example for \Tmc iff $(\Tmc,\Amc)\models q$.
We denote these learning frameworks with $\Fmf(\Lmf,IQ)$,
where \Lmf is the DL. 
In both problem settings, if the return of an equivalence query
is `yes' then this means that the hypothesis of the learner
and the target are logically equivalent.
Recently, it has been shown that if the ABox is fixed
and one only aims at preserving IQ results w.r.t. the fixed ABox
(not  logical equivalence between the hypothesis and the target)
then there is a polynomial time algorithm for \ELH terminologies~\cite{DBLP:journals/corr/abs-1911-07229}. 
We denote this learning framework by $\Fmf(\ELH,\Amc,IQ)$ where
$\Amc$ is the fixed ABox. The intuition for why the problem is `easier'
in this case is because, since the ABox is fixed, the possible counterexamples the teacher can give
are constrained.
The fixed ABox setting avoids the difficult scenario described
in the hardness proof for the learning framework $\Fmf(\ELH,IQ)$~\cite[Page 29 of the full version]{DBLP:conf/aaai/KonevOW16}, where the
teacher can give counterexamples of the form $(\Amc_{\sma},A(a))$,
with $\sma = \sigma^1,\ldots, \sigma^n$, for $\sigma^i\in\{A_i,\overline{A_i}\}$, and
$\Amc_{\sma}$  an ABox
of the form  $\{\sigma^1(a),\ldots,\sigma^n(a)\}$.
%

By Theorem~\ref{thm:exact-pac},  positive results in the exact learning model are transferable
to the PAC model extended with membership queries.
We point out that the complexity class $\PTimePL({\sf SQ})$ is \emph{not}
contained in $\PTimeL({\sf MQ},{\sf EQ})$. This has been
discovered by Blum in 1994~\cite{Blum:1994:SDM:196751.196815}.
He constructed an artificial counterexample to prove the result and the argument
relies on cryptographic assumptions.
Another (artificial) counterexample appears in the work by Ozaki et al. (2020)~\cite{DBLP:journals/corr/abs-1911-07229}.
The argument in this case does not rely on cryptographic assumptions.
Apart from these carefully constructed learning frameworks, in many cases,
learning frameworks in $\PTimePL({\sf SQ})$ are also in 
$\PTimeL({\sf MQ},{\sf EQ})$.

We now explain why  $\Fmf(\ELH)$ and $\Fmf(\ELH,IQ)$ appear
in $\ExpTimeL({\sf MQ},{\sf EQ})$. In fact,
they are already  in $\ExpTimeL({\sf EQ})$.
The reason is that, as explained at the end of Subsection~\ref{subsec:alg},
 since 
the learning system  receives the (finite) signature $\Sigma_\Tmc$
of the target \Tmc as input,
 it can enumerate all TBoxes (up to logical equivalence)
of a certain size
and ask whether any of them is equivalent to  \Tmc, one by one, increasing
this size until a TBox equivalent to \Tmc is found. 
This naive procedure clearly requires an
exponential number of steps in the size of \Tmc.
The same
holds for other DL languages more expressive
than \ELH, as long as TBoxes can also be enumerated in this way. 
An exponential (but non-trivial) algorithm for \EL terminologies and its implementation
   is provided by Duarte et al. (2018)~\cite{DBLP:conf/kr/DuarteKO18}

It remains to explain the $\Fmf(\text{DL-Lite}_\Rmc)$ case. 
$\text{DL-Lite}_{\Rmc}$ 
is a member of a well-known family of DLs~\cite{DBLP:journals/corr/ArtaleCKZ14}.
What we would like to explain is that, for some ontology languages,
such as $\text{DL-Lite}_{\Rmc}$,
the number of RIs and CIs that can
be formulated within the  (finite) signature $\Sigma_\Tmc$ of a target \Tmc
is polynomial in the size of $\Sigma_\Tmc$.
Since $\Sigma_\Tmc$ is given as input to the learning system,
this means that the learner can identify the target
with only membership queries (see  toy example in Subsection~\ref{subsec:alg}), and moreover, 
in polynomial
time in the size of $\Sigma_\Tmc$.
Therefore, $\Fmf(\text{DL-Lite}_\Rmc)$ belongs to $\PTimeL({\sf MQ})$.
Learning an equivalent $\text{DL-Lite}_{\Rmc}$ TBox with only
equivalence queries is also easy.
This happens because there are only polynomially many
counterexamples that can be given.
The learner can start by posing an equivalence
query with an empty hypothesis.
Then, the teacher is obliged to return a positive counterexample
(unless the target
is  equivalent to the empty hypothesis and we are done).
All the learner needs to do is to add this positive counterexample to
its hypothesis and then proceed by posing another equivalence query.
After polynomially many equivalence queries, the learner
will terminate with an equivalent hypothesis.
So $\Fmf(\text{DL-Lite}_\Rmc)\in\PTimeL({\sf EQ})$ also holds.






\section{Related Work}

We now highlight 
some other approaches from the literature for learning DL ontologies,
when the focus is on finding how terms of an ontology
should relate to each other using the expressivity of the ontology language at hand. 
These approaches are mainly based on association rule mining, formal concept analysis,
inductive logic programming, and neural networks~\cite{anasurvey}.


Formal concept analysis~\cite{FCA} has been applied to mine \EL CIs~\cite{BorDi11} (see also~\cite{Rudolph04exploringrelational,baader2007completing,DBLP:conf/icfca/BaaderD09}).
In this setting,
a learner receives a finite intepretation \Imc as input
and attempts to build a finite ontology \Tmc such that,
for all \EL CIs formulated in a finite signature, 
$\Tmc\models C\sqsubseteq D$ if, and only if, $\Imc\models C\sqsubseteq D$. 
This ontology, called \emph{base}, should also satisfy certain minimality conditions.
It is known that, given a finite interpretation \Imc,
a finite base (expressed within a finite signature) always exists for the \EL ontology language. 
However, this may not be the case for other ontology languages.
%
The main difficulty in applying formal concept analysis for
building ontologies is that, as originally proposed,
it cannot build CIs one may expect to hold when the data is (even just slightly) incorrect.
If a certain CI holds in practice but there is an element that
violates it in the interpretation then the CI will not be included
in the base. One could argue that this method is then useful to find such errors but
the application for mining CIs has this issue.

When there is a certain threshold for tolerating errors, then association rule
mining offers an interesting solution. This method is based on the measures of
support and confidence~\cite{agrawal1993mining}. The support is a metric for measuring statistical significance,
while confidence measures the ‘strength’ of a rule, in this case, expressed
as a CI in an ontology language. Many authors have already employed this method
for building DL ontologies~\cite{volker2011statistical,DBLP:conf/otm/FleischhackerVS12,DBLP:journals/ws/VolkerFS15}
(see also~\cite{DBLP:conf/semweb/SazonauS17}) and for finding relational rules
 in knowledge 
graphs~\cite{DBLP:journals/vldb/GalarragaTHS15}.
The usual approach 
is to fix the depth of the CIs in order to restrict the search space. 

There is a vast literature on algorithms and techniques for learning DL \emph{concepts} based on inductive logic programming~\cite{DBLP:conf/ijcai/FunkJLPW19,DBLP:conf/ilp/FanizzidE08,DBLP:journals/apin/IannonePF07,lehmann2009dl,lehmann2009ideal,
DBLP:journals/ml/LehmannH10,lehmann2010learning}(see~\cite{DBLP:journals/ijswis/Lisi11} for learning logical rules also
based in inductive logic programming).
One of the most well known tools for supporting the construction of DL concepts is the DL-Learner~\cite{lehmann2009dl}.
In this approach, the learner receives as input examples of assertions classified
as positive and negative and the goal is to construct a DL concept expression
that `fits' the classified examples. 

Deep learning has also been applied for learning DL ontologies~\cite{Petrucci:2016:OLD:3092960.3092992}.
In the mentioned work, the authors use definitional sentences labelled with
their DL translation to train a recurrent neural network (see also~\cite{DBLP:conf/aime/MaD13} for 
more work on definitional sentences in a DL context). It is an interesting
approach that deals extremely well with data variability.
The main difficulties pointed out by the authors are how to find large amounts of
classified examples to train the neural network (the authors have trained it using
synthetic data) and how to capture the semantics of the ontology. The
neural network could capture the syntax, for example, map the word `and' to the logical
operator  `$\sqcap$'. However, as reported by the authors,
the method does not really capture the semantics of the sentences and how they relate to each other. 
There is an extensive literature on learning assertions using neural networks~\cite{DBLP:conf/nips/BordesUGWY13,DBLP:journals/corr/YangYHGD14a} but not many works on
building DL ontologies with complex concept expressions. 


\section{Conclusion} 

We have presented a formalisation of the exact and the PAC learning
models and defined learning complexity classes.
This opens the possibility of investigating other
questions such as the problem of deciding whether a learning framework
is PAC or exactly learnable. 
%
Some authors
have already investigated 
 this problem for the PAC model,
with a different formalisation of PAC learnability~\cite{DBLP:journals/corr/abs-1808-06324}.

An interesting application of exact learning algorithms is 
to verify neural networks, as in the already mentioned works by Weiss et al.~\cite{DBLP:conf/nips/WeissGY19,DBLP:conf/icml/WeissGY18}. 
These works are based on Angluin's exact learning algorithm for
learning regular languages represented by deterministic finite automata
with membership and equivalence queries (an abstraction of the automata is used to find counterexamples). 
One of the goals of this strategy is to 
find \emph{adversarial inputs}: examples neither present
in the training nor in the test set which were misclassified
by a neural network~\cite{DBLP:conf/nips/WeissGY19,DBLP:conf/icml/WeissGY18}. 
It would be interesting to investigate whether algorithms for exactly
learning ontologies can also be applied to verify if a neural network
captures certain rules. 



 
%
%
%
\bibliographystyle{spmpsci}      

\bibliography{references}

\end{document}

%% file: tik.tex
\pgfdeclarelayer{edgelayer}
\pgfdeclarelayer{nodelayer}
\pgfsetlayers{edgelayer,nodelayer,main}

\tikzstyle{none}=[inner sep=0pt]
\tikzstyle{invisible}=[inner sep=0pt]
\tikzstyle{rn}=[circle,fill=Red,draw=Black,line width=0.8 pt]
\tikzstyle{gn}=[circle,fill=White,draw=Black,line width=0.8 pt]
\tikzstyle{yn}=[circle,fill=Yellow,draw=Black,line width=0.8 pt]
\tikzstyle{simple}=[circle,fill=White,draw=Black]
\tikzstyle{newstyle1}=[circle,fill=Black,draw=Black,line width=0.3 pt,inner sep=0pt]
\tikzstyle{simple2}=[-,dashed,draw=Black]
\tikzstyle{simpledotted}=[-,dotted,draw=Black]
\tikzstyle{simple}=[-,draw=Black,line width=2.000]
\tikzstyle{new}=[-,draw=Black,line width=2.000]
\tikzstyle{arrow}=[-,draw=Black,postaction={decorate},decoration={markings,mark=at position .5 with {\arrow{>}}},line width=2.000]
\tikzstyle{tick}=[-,draw=Black,postaction={decorate},decoration={markings,mark=at position .5 with {\draw (0,-0.1) -- (0,0.1);}},line width=2.000]
\tikzstyle{newstyle2}=[-latex,draw=Black]
\tikzstyle{newstyle3}=[->,dotted,draw=Black]
\tikzstyle{newstyle6}=[->,dotted,draw=Black]

%% file: misc0.tex
\usepackage{amsmath,amssymb}

\usepackage{amsthm}
\usepackage{hyperref}
\usepackage{xspace}
\usepackage{stmaryrd}
\usepackage{wasysym}
\usepackage{graphicx}
\usepackage{microtype}
\usepackage{cite}
\usepackage{array}
\usepackage[colorinlistoftodos,obeyFinal]{todonotes}
\usepackage{booktabs}
\usepackage{csquotes}
\usepackage{enumitem}

\usepackage{tikz}
\usetikzlibrary{backgrounds}
\usetikzlibrary{calc}
\usetikzlibrary{positioning,fit}
\usepgfmodule{plot}

\def\rConstrHelper(#1,#2,#3,#4){{#1}\ \substack{#2 \\ #3}\ {#4}}

\newcommand{\prob}{\ensuremath{\Dmc}\xspace}

\newcommand{\poly}{\ensuremath{\mathit{p}}\xspace}
\usepackage{scalerel,stackengine}

\newcommand{\sss}{{\mathrel{\kern.25em{\sqsubseteq}\kern-.5em \mbox{{\scriptsize *}}\kern.25em}}}
\newcommand{\smallsss}{{\mathrel{\kern.25em{\sqsubseteq}\kern-.4em \mbox{{\scriptsize *}}\kern.25em}}}

\newcommand{\node}{\ensuremath{\mathop{\mathsf{node}}}}

\newcommand{\ELHlhs}{\ensuremath{{\cal E\!LH}_{lhs}}\xspace}
\newcommand{\ELHrhs}{\ensuremath{{\cal E\!LH}_{rhs}}\xspace}

\newcommand{\ELH}{\ensuremath{{\cal E\!LH}}\xspace}

\newcommand{\EL}{\ensuremath{{\cal E\!L}}\xspace}

\newcommand{\NC}{\ensuremath{{\sf N_C}}\xspace}

\newcommand{\NI}{\ensuremath{{\sf N_I}}\xspace}

\newcommand{\NR}{\ensuremath{{\sf N_R}}\xspace}

\newlength{\indxlength}
\newcommand{\ExpTimeL}{\textsc{ElExp}\xspace}
\newcommand{\PTimeL}{\textsc{ElP}\xspace}
\newcommand{\PTimePL}{\textsc{PlP}\xspace}
\newcommand{\PL}{\textsc{Pl}\xspace}

\newcommand{\XL}{\textsc{El}\xspace}
 
\newcommand{\Amc}{\ensuremath{\mathcal{A}}\xspace}

\newcommand{\Dmc}{\ensuremath{\mathcal{D}}\xspace}
\newcommand{\Emc}{\ensuremath{\mathcal{E}}\xspace}

\newcommand{\Hmc}{\ensuremath{\mathcal{H}}\xspace}
\newcommand{\Imc}{\ensuremath{\mathcal{I}}\xspace}

\newcommand{\Lmc}{\ensuremath{\mathcal{L}}\xspace}
\newcommand{\Mmc}{\ensuremath{\mathcal{M}}\xspace}

\newcommand{\Pmc}{\ensuremath{\mathcal{P}}\xspace}

\newcommand{\Rmc}{\ensuremath{\mathcal{R}}\xspace}

\newcommand{\Tmc}{\ensuremath{\mathcal{T}}\xspace}

\newcommand{\Fmf}{\ensuremath{\mathfrak{F}}\xspace}

\newcommand{\Lmf}{\ensuremath{\mathfrak{L}}\xspace}



%% file: learningsystem.tikz
\begin{tikzpicture}
		\node [style=none] (0) at (-6, 3.25) {$\learnertm$};
		\node [style=none] (1) at (-2, 3.25) {$\teachertm(\Tmc)$};
		\node [style=none] (2) at (-6, 2.5) {\scriptsize{Input Tape}};
		\node [style=none] (3) at (-2, 2.5) {\scriptsize{Oracle Tape}};
		\node [style=none] (4) at (-7, 2.25) {};
		\node [style=none] (5) at (-5, 2.25) {};
		\node [style=none] (6) at (-7, 1.75) {};
		\node [style=none] (7) at (-5, 1.75) {};
		\node [style=none] (8) at (-6, 2) {$\Sigma_\Tmc$};
		\node [style=none] (9) at (-1, 2.25) {};
		\node [style=none] (10) at (-3, 1.75) {};
		\node [style=none] (11) at (-3, 2.25) {};
		\node [style=none] (12) at (-1, 1.75) {};
		\node [style=none] (13) at (-2, 2) {$\Tmc$};
		\node [style=none] (14) at (-3, 1) {};
		\node [style=none] (15) at (-5, 1) {};
		\node [style=none] (16) at (-5, 0.5) {};
		\node [style=none] (17) at (-3, 0.5) {};
		\node [style=none] (18) at (-4, 0.75) {$A \sqsubseteq B$};
		\node [style=none] (19) at (-5.75, 1.5) {};
		\node [style=none] (20) at (-5.25, 0.75) {};
		\node [style=none] (21) at (-2.75, 0.75) {};
		\node [style=none] (22) at (-2.25, 1.5) {};
		\node [style=none] (23) at (-6, 1) {\scriptsize{writes}};
		\node [style=none] (24) at (-2, 1) {\scriptsize{reads}};
		\node [style=none] (25) at (-5, -0.75) {};
		\node [style=none] (26) at (-4, -0.5) {yes};
		\node [style=none] (27) at (-6.75, 1.25) {};
		\node [style=none] (28) at (-1.25, -0.25) {\scriptsize{writes}};
		\node [style=none] (29) at (-1.25, 1.25) {};
		\node [style=none] (30) at (-5.25, -0.5) {};
		\node [style=none] (31) at (-3, -0.25) {};
		\node [style=none] (32) at (-3, -0.75) {};
		\node [style=none] (33) at (-6.5, -0.25) {\scriptsize{reads}};
		\node [style=none] (34) at (-2.75, -0.5) {};
		\node [style=none] (35) at (-5, -0.25) {};
		\node [style=none] (36) at (-4, 1.25) {\scriptsize{Communication Tape}};
		\draw [style=dashed] (6.center) to (7.center);
		\draw [style=dashed] (7.center) to (5.center);
		\draw [style=dashed] (5.center) to (4.center);
		\draw [style=dashed] (4.center) to (6.center);
		\draw [style=dashed] (10.center) to (12.center);
		\draw [style=dashed] (12.center) to (9.center);
		\draw [style=dashed] (9.center) to (11.center);
		\draw [style=dashed] (11.center) to (10.center);
		\draw [style=dashed] (16.center) to (17.center);
		\draw [style=dashed] (17.center) to (14.center);
		\draw [style=dashed] (14.center) to (15.center);
		\draw [style=dashed] (15.center) to (16.center);
		\draw [style=newstyle2] (19.center) to (20.center);
		\draw [style=newstyle2] (21.center) to (22.center);
		\draw [style=dashed] (25.center) to (32.center);
		\draw [style=dashed] (32.center) to (31.center);
		\draw [style=dashed] (31.center) to (35.center);
		\draw [style=dashed] (35.center) to (25.center);
		\draw [style=newstyle2, bend left, looseness=1.00] (30.center) to (27.center);
		\draw [style=newstyle2, bend left=45, looseness=1.00] (29.center) to (34.center);
\end{tikzpicture}